\documentclass[12pt]{article}%
\usepackage{amssymb}
\usepackage{amsfonts}
\usepackage{graphicx, epsfig}
\usepackage{amsmath}
\usepackage[hidelinks]{hyperref}
\usepackage{graphicx}
\usepackage{color}
\usepackage{rotating}
\usepackage[authoryear]{natbib}
\usepackage[left=1.3cm, right=1.3cm]{geometry}
\usepackage{hyperref}%
\usepackage{authblk}
\usepackage{placeins}
\usepackage{multirow}

\newtheorem{proposition}{Proposition}

\newenvironment{proof}[1][Proof]{\textbf{#1.} }{\ \rule{0.5em}{0.5em}}

\begin{document}

\title{KL Divergence Between Gaussians: A Step-by-Step Derivation for the Variational Autoencoder Objective}

\author{
Andrés Muñoz \and Rodrigo Ramele \\
Instituto Tecnológico de Buenos Aires
}

\date{}
\maketitle
\begin{abstract}
Kullback-Leibler (KL) divergence is a fundamental concept in information theory that quantifies the discrepancy between two probability distributions. In the context of Variational Autoencoders (VAEs), it serves as a central regularization term, imposing structure on the latent space and thereby enabling the model to exhibit generative capabilities.
In this work, we present a detailed derivation of the closed-form expression for the KL divergence between Gaussian distributions—a case of particular importance in practical VAE implementations. Starting from the general definition for continuous random variables, we derive the expression for the univariate case and extend it to the multivariate setting under the assumption of diagonal covariance. Finally, we discuss the interpretation of each term in the resulting expression and its impact on the training dynamics of the model.
\end{abstract}


\section{Introduction}
\label{sec:intro}

Probabilistic modeling has become a cornerstone of modern machine learning, particularly in the development of generative models capable of capturing complex data distributions. Among these, Variational Autoencoders (VAEs), introduced by   \cite{kingma2013auto} have emerged as a powerful framework that integrates principles from deep learning and variational inference to learn meaningful latent representations. 

A key component of the VAE formulation is the incorporation of a structured latent space, typically enforced by assuming a prior distribution over the latent variables. In practice, this prior is commonly chosen as a standard normal distribution. To ensure that the learned representations adhere to this prior, VAEs incorporate a regularization term based on the Kullback--Leibler (KL) divergence, which measures the discrepancy between the approximate posterior distribution produced by the encoder and the prior distribution.

Despite its widespread use, the KL divergence is often introduced in an abstract manner, and its closed-form expression for Gaussian distributions is frequently presented without a detailed derivation. However, understanding this derivation is essential for gaining deeper insight into the behavior of VAEs, particularly regarding how the latent space is shaped during training.

In this work, we provide a step-by-step derivation of the KL divergence between Gaussian distributions, starting from its general definition for continuous random variables. We first consider the univariate case and then extend the result to the multivariate setting under the assumption of diagonal covariance. Finally, we discuss the interpretation of the resulting expression and its implications for the training dynamics and regularization properties of VAEs.

\section{Kullback--Leibler Divergence}

The Kullback-Leibler (KL) divergence, originally introduced in \cite{kullback1951information}, is a fundamental measure in information theory that quantifies the discrepancy between two probability distributions. Let $P$ and $Q$ be two probability distributions defined over the same space $\mathcal{X}$, with corresponding densities $p$ and $q$.   The KL divergence of $P$ with respect to $Q$ is defined as: 

\begin{align}\label{def_KL} 
D_{KL}(P \| Q) = \int_{\mathcal{X}} p(x)\, \log \left( \frac{p(x)}{q(x)} \right)\, dx
\end{align}

\noindent provided that $p(x) = 0$ whenever $q(x) = 0$, ensuring that the expression is well-defined.

Equivalently, the KL divergence can be expressed as an expectation with respect to $P$:

\begin{align*}
D_{KL}(P \| Q) = \mathbb{E}_{ P} \left[ \log \frac{p(x)}{q(x)} \right]
\end{align*}

The KL divergence satisfies the following properties:

\begin{itemize}
\item $D_{KL}(P \| Q) \geq 0$ \quad (non-negativity)
\item $D_{KL}(P \| Q) = 0$ if and only if $P = Q$ (almost everywhere)
\item $D_{KL}(P \| Q) \neq D_{KL}(Q \| P)$ in general (asymmetry)
\end{itemize}

Despite being commonly referred to as a “distance”, the KL divergence is not a true metric, as it does not satisfy symmetry nor the triangle inequality. Instead, it can be interpreted as the expected information loss incurred when using $Q$ to approximate $P$.

\section{Derivation for Gaussian Distributions}\label{sec:propuesta}

Let $P$ and $Q$ be two probability distributions such that
\[
P \sim \mathcal{N}(\mu_1, \Sigma_1), 
\qquad 
Q \sim \mathcal{N}(\mu_2, \Sigma_2),
\]
with $\mu_1, \mu_2 \in \mathbb{R}^k$ and $\Sigma_1, \Sigma_2 \in \mathbb{R}^{k \times k}$ and let $p$ and $q$ denote their corresponding probability density functions 

\begin{equation}\label{dens_norm}
\begin{aligned}
p(x) &= \frac{1}{\sqrt{(2\pi)^k |\Sigma_1|}} 
\exp\!\left( -\frac{1}{2}(x-\mu_1)^\top \Sigma_1^{-1} (x-\mu_1) \right) \\[6pt]
q(x) &= \frac{1}{\sqrt{(2\pi)^k |\Sigma_2|}} 
\exp\!\left( -\frac{1}{2}(x-\mu_2)^\top \Sigma_2^{-1} (x-\mu_2) \right)
\end{aligned} 
\end{equation}  By a slight abuse of notation, we will use $p$ and $q$ interchangeably to refer to the corresponding distributions when no ambiguity arises.

In the following, we assume that $\mathcal{X} = \mathbb{R}^k$. For notational simplicity, the domain of integration will be omitted. Starting from Definition (\ref{def_KL})  and applying basic properties of the logarithm, the Kullback - Leibler divergence can be rewritten as:
\begin{align}\label{KL_log}
D_{KL}(P \| Q)= \int p(x)\bigl(\log p(x) - \log q(x)\bigr)\, dx
\end{align}

The logarithms of each density  (\ref{dens_norm})  are:

\begin{align*}
\log p(x) &= -\frac{k}{2}\log(2\pi) - \frac{1}{2}\log|\Sigma_1| 
             - \frac{1}{2}(x-\mu_1)^\top \Sigma_1^{-1}(x-\mu_1) \\[4pt]
\log q(x) &= -\frac{k}{2}\log(2\pi) - \frac{1}{2}\log|\Sigma_2| 
             - \frac{1}{2}(x-\mu_2)^\top \Sigma_2^{-1}(x-\mu_2)
\end{align*}

Substituting into (\ref{KL_log}) the KL divergence (noting that the constant terms cancel), we obtain:

\begin{align*}
D_{KL}(P \| Q)
  = \int p(x) \Biggl[
      \frac{1}{2}\log\frac{|\Sigma_2|}{|\Sigma_1|}
      + \frac{1}{2}(x-\mu_2)^\top \Sigma_2^{-1}(x-\mu_2)
      - \frac{1}{2}(x-\mu_1)^\top \Sigma_1^{-1}(x-\mu_1)
    \Biggr] dx
\end{align*}

Using the linearity of the integral, we decompose the expression into three terms, which we denote by $H_1$, $H_2$, and $H_3$.

\begin{align}
D_{KL}(P \| Q)
&= \int p(x)\,\frac{1}{2}\log\frac{|\Sigma_2|}{|\Sigma_1|}\, dx  \nonumber \\ 
&\quad + \int p(x)\,\frac{1}{2}(x-\mu_2)^\top \Sigma_2^{-1}(x-\mu_2)\, dx  \nonumber \\
&\quad - \int p(x)\,\frac{1}{2}(x-\mu_1)^\top \Sigma_1^{-1}(x-\mu_1)\, dx  \nonumber \\[6pt]
&= H_1 + H_2 - H_3   \label{KL}
\end{align}

We now compute each term separately.

\paragraph{Term $H_1$.}

Since the logarithmic term is constant with respect to $x$, and recalling that
\[
\int p(x)\,dx = 1,
\]
we obtain:

\[
H_1 = \frac{1}{2}\log\frac{|\Sigma_2|}{|\Sigma_1|}.
\]

\paragraph{Term $H_2$.}

We consider the term
\[
H_2 = \mathbb{E}_P\!\left[(x-\mu_2)^\top \Sigma_2^{-1}(x-\mu_2)\right].
\]

Rewriting the difference as
\[
x - \mu_2 = (x - \mu_1) + (\mu_1 - \mu_2),
\]
we expand the quadratic form:

\begin{align*}
(x-\mu_2)^\top \Sigma_2^{-1}(x-\mu_2)
&= \big[(x-\mu_1) + (\mu_1 - \mu_2)\big]^\top \Sigma_2^{-1} \big[(x-\mu_1) + (\mu_1 - \mu_2)\big] \\
&= (x-\mu_1)^\top \Sigma_2^{-1}(x-\mu_1) \\
&\quad + 2(\mu_1 - \mu_2)^\top \Sigma_2^{-1}(x-\mu_1) \\
&\quad + (\mu_1 - \mu_2)^\top \Sigma_2^{-1}(\mu_1 - \mu_2)
\end{align*}

Taking expectation with respect to $P$, we analyze each term separately.

\medskip

\noindent\textit{First term.}
Using the identity $\mathbb{E}[x^\top A x] = \mathrm{tr}(A\,\mathbb{E}[x x^\top])$, applied to the centered variable $(x-\mu_1)$, and recalling that the covariance matrix is defined as
\[
\Sigma_1 = \mathbb{E}\left[(x-\mu_1)(x-\mu_1)^\top\right],
\]
we obtain:
\begin{align*}
\mathbb{E}_P\!\left[(x-\mu_1)^\top \Sigma_2^{-1}(x-\mu_1)\right]
&= \mathrm{tr}\!\left(\Sigma_2^{-1}\,\mathbb{E}_P[(x-\mu_1)(x-\mu_1)^\top]\right) \\
&= \mathrm{tr}\!\left(\Sigma_2^{-1} \Sigma_1\right)
\end{align*}
\medskip

\noindent\textit{Second term.}
Since $\mathbb{E}_P[x] = \mu_1$ and $(\mu_1 - \mu_2)^\top \Sigma_2^{-1}$ is constant with respect to $x$, we obtain:
\begin{align*}
\mathbb{E}_P\!\left[(\mu_1 - \mu_2)^\top \Sigma_2^{-1}(x-\mu_1)\right]
&= (\mu_1 - \mu_2)^\top \Sigma_2^{-1}\,\mathbb{E}_P[x-\mu_1] =0\\
\end{align*}

\medskip

\noindent\textit{Third term.}
Since this term is constant with respect to $x$, we have:
\begin{align*}
\mathbb{E}_P\!\left[(\mu_1 - \mu_2)^\top \Sigma_2^{-1}(\mu_1 - \mu_2)\right]
= (\mu_1 - \mu_2)^\top \Sigma_2^{-1}(\mu_1 - \mu_2)
\end{align*}

\medskip

\noindent Combining the three terms, we obtain:
\begin{align*}
H_2
&= \mathrm{tr}\!\left(\Sigma_2^{-1} \Sigma_1\right)
+ (\mu_1 - \mu_2)^\top \Sigma_2^{-1} (\mu_1 - \mu_2)
\end{align*}

\paragraph{Term $H_3$.} We consider the term

\[
 H_3 = \mathbb{E}_P\!\left[  \frac{1}{2} (x-\mu_1)^\top \Sigma_1^{-1}(x-\mu_1)\right]
\]

Using the identity $\mathbb{E}[x^\top A x] = \mathrm{tr}(A\,\mathbb{E}[x x^\top])$, and recalling that the covariance matrix is defined as
\[
\Sigma_1 = \mathbb{E}\left[(x-\mu_1)(x-\mu_1)^\top\right],
\]
we obtain:

\begin{align*}
H_3
&=\frac{1}{2} \mathrm{tr}\!\left(\Sigma_1^{-1} \Sigma_1\right) \\
&= \frac{1}{2}  \mathrm{tr}(I_k) \\
&= \frac{1}{2} k
\end{align*}

\paragraph{Final expression.}

Substituting $H_1$, $H_2$, and $H_3$ into the expression  (\ref{KL}) for the KL divergence, we obtain:

\begin{align*}
D_{KL}(P \| Q)
&= H_1 + H_2 - H_3 \\[6pt]
&= \frac{1}{2}\log\frac{|\Sigma_2|}{|\Sigma_1|}
+ \frac{1}{2} \left[
\mathrm{tr}(\Sigma_2^{-1} \Sigma_1)
+ (\mu_1 - \mu_2)^\top \Sigma_2^{-1} (\mu_1 - \mu_2)
\right]
- \frac{1}{2}k
\end{align*}

\noindent Rearranging terms we obtein:

\begin{align*}
D_{KL}(P \| Q)
= \frac{1}{2} \left(
\mathrm{tr}(\Sigma_2^{-1} \Sigma_1)
+ (\mu_1 - \mu_2)^\top \Sigma_2^{-1} (\mu_1 - \mu_2)
- k
+ \log\frac{|\Sigma_2|}{|\Sigma_1|}
\right)
\end{align*}

\begin{proposition}
Let $P \sim \mathcal{N}(\mu_1, \Sigma_1)$ and $Q \sim \mathcal{N}(\mu_2, \Sigma_2)$ be two multivariate Gaussian distributions in $\mathbb{R}^k$, where $\mu_1, \mu_2 \in \mathbb{R}^k$ are the mean vectors, and $\Sigma_1, \Sigma_2 \in \mathbb{R}^{k \times k}$ are the corresponding covariance matrices, assumed to be symmetric and positive definite.

\noindent Then, the Kullback--Leibler divergence between $P$ and $Q$ is given by:
\begin{align*}
D_{KL}(P \| Q)
= \frac{1}{2} \left(
\mathrm{tr}(\Sigma_2^{-1} \Sigma_1)
+ (\mu_1 - \mu_2)^\top \Sigma_2^{-1} (\mu_1 - \mu_2)
- k
+ \log\frac{|\Sigma_2|}{|\Sigma_1|}
\right).
\end{align*}
\end{proposition}

\noindent \begin{proof}
See the derivation above.
\end{proof}

\paragraph{Special case (VAE).} This particular case is of special interest in the context of Variational Autoencoders (VAEs), where the latent variables are typically modeled using a Gaussian approximate posterior and a standard normal prior. Under these assumptions, the KL divergence admits a closed-form expression, making it computationally efficient to evaluate and differentiate during training. This property is crucial for enabling scalable optimization via gradient-based methods and for enforcing a well-structured latent space.

\noindent In this setting, the approximate posterior is given by
\[
q(z|x) \sim \mathcal{N}(\mu(x), \Sigma(x)),
\]
while the prior is defined as
\[
p(z) \sim \mathcal{N}(0, I_k).
\]

\noindent Therefore, the KL divergence simplifies to:
\begin{align*}
D_{KL}(q(z|x) \| p(z))
= \frac{1}{2} \left(
\mathrm{tr}(\Sigma(x)) 
+ \mu(x)^\top \mu(x) 
- k 
- \log |\Sigma(x)|
\right)
\end{align*}

\noindent This expression is commonly used as the regularization term in the VAE objective function.

\section{Algebraic Identities}

In this appendix, we summarize several algebraic identities used throughout the derivation.

\paragraph{Trace properties.}

Let $A, B$ be matrices of compatible dimensions. Then:

\begin{itemize}
\item $\mathrm{tr}(A) = \mathrm{tr}(A^\top)$
\item $\mathrm{tr}(AB) = \mathrm{tr}(BA)$
\item $\mathrm{tr}(ABC) = \mathrm{tr}(BCA) = \mathrm{tr}(CAB)$
\item $\mathrm{tr}(aA + bB) = a\,\mathrm{tr}(A) + b\,\mathrm{tr}(B)$ for scalars $a,b$
\end{itemize}

\paragraph{Quadratic forms.}

Let $x \in \mathbb{R}^k$ and $A \in \mathbb{R}^{k \times k}$. Then:

\begin{itemize}
\item $x^\top A x = \mathrm{tr}(x^\top A x) = \mathrm{tr}(A x x^\top)$
\end{itemize}

\paragraph{Expectation and trace.}

For a random vector $x \in \mathbb{R}^k$:

\begin{itemize}
\item $\mathbb{E}[\mathrm{tr}(X)] = \mathrm{tr}(\mathbb{E}[X])$
\item $\mathbb{E}[x^\top A x] = \mathrm{tr}(A\,\mathbb{E}[x x^\top])$
\end{itemize}

\paragraph{Covariance matrix.}

For a random vector $x$ with mean $\mu$:

\[
\Sigma = \mathbb{E}\big[(x - \mu)(x - \mu)^\top\big]
\]

\paragraph{Gaussian expectations.}

If $x \sim \mathcal{N}(\mu, \Sigma)$, then:

\begin{itemize}
\item $\mathbb{E}[x] = \mu$
\item $\mathbb{E}[(x - \mu)(x - \mu)^\top] = \Sigma$
\item $\mathbb{E}[x x^\top] = \Sigma + \mu \mu^\top$
\end{itemize}

\bibliographystyle{apalike}
\bibliography{biblio}

\end{document}